%% file: main.tex
\documentclass[english,british]{article}
\usepackage[T1]{fontenc}
\usepackage[latin9]{inputenc}
\usepackage[active]{srcltx}
\usepackage{color}
\usepackage{url}
\usepackage{algorithm2e}
\usepackage{amsmath}
\usepackage{amsthm}
\usepackage{amssymb}
\usepackage{graphicx}
\usepackage[authoryear]{natbib}

\makeatletter
\theoremstyle{definition}
\newtheorem{defn}{\protect\definitionname}
\theoremstyle{plain}
\newtheorem{lem}{\protect\lemmaname}
\ifx\proof\undefined\
  \newenvironment{proof}[1][\proofname]{\par
    \normalfont\topsep6\p@\@plus6\p@\relax
    \trivlist
    \itemindent\parindent
    \item[\hskip\labelsep
          \scshape
      #1]\ignorespaces
  }{%
    \endtrivlist\@endpefalse
  }
  \providecommand{\proofname}{Proof}
\fi
\theoremstyle{plain}
\newtheorem{thm}{\protect\theoremname}

\usepackage{aaai23}  

\usepackage{amsthm}

\usepackage{caption}

\usepackage{booktabs}
\newcommand{\indep}{\perp \!\!\! \perp}

\newcommand{\argmax}{\arg\,\max}

\author{
Bao Duong, Thin Nguyen
}

\affiliations {
Applied Artificial Intelligence Institute, Deakin University, Australia
}

\makeatother

\usepackage{babel}
\addto\captionsbritish{\renewcommand{\definitionname}{Definition}}
\addto\captionsbritish{\renewcommand{\lemmaname}{Lemma}}
\addto\captionsbritish{\renewcommand{\theoremname}{Theorem}}
\addto\captionsenglish{\renewcommand{\definitionname}{Definition}}
\addto\captionsenglish{\renewcommand{\lemmaname}{Lemma}}
\addto\captionsenglish{\renewcommand{\theoremname}{Theorem}}
\providecommand{\definitionname}{Definition}
\providecommand{\lemmaname}{Lemma}
\providecommand{\theoremname}{Theorem}

\begin{document}
\title{Diffeomorphic Information Neural Estimation}
\maketitle
\begin{abstract}
\input{abs.tex}
\end{abstract}

\section{Introduction}

\input{intro.tex}

\section{Background}

\input{related.tex}

\section{The Diffeomorphic Information Neural Estimator (DINE)}

\input{method.tex}

\section{Theoretical Properties}

\input{theory.tex}

\section{Empirical Evaluations}

\input{exp.tex}

\section{Application in Conditional Independence Testing}

\input{app.tex}

\section{Conclusion}

\input{conclude.tex}

\bibliography{references}

\appendix
\input{Supplementary.tex}

\end{document}

%% file: abs.tex
Mutual Information (MI) and Conditional Mutual Information (CMI) are
multi-purpose tools from information theory that are able to naturally
measure the statistical dependencies between random variables, thus
they are usually of central interest in several statistical and machine
learning tasks, such as conditional independence testing and representation
learning. However, estimating CMI, or even MI, is infamously challenging
due the intractable formulation. In this study, we introduce \textbf{DINE}
(\textbf{D}iffeomorphic \textbf{I}nformation \textbf{N}eural \textbf{E}stimator)--a
novel approach for estimating CMI of continuous random variables,
inspired by the invariance of CMI over diffeomorphic maps. We show
that the variables of interest can be replaced with appropriate surrogates
that follow simpler distributions, allowing the CMI to be efficiently
evaluated via analytical solutions. Additionally, we demonstrate the
quality of the proposed estimator in comparison with state-of-the-arts
in three important tasks, including estimating MI, CMI, as well as
its application in conditional independence testing. The empirical
evaluations show that \textbf{DINE} consistently outperforms competitors
in all tasks and is able to adapt very well to complex and high-dimensional
relationships.

%% file: intro.tex
\selectlanguage{english}%
Mutual Information (MI) and Conditional Mutual Information (CMI) are
pivotal dependence measures between random variables for general non-linear
relationships. In statistics and machine learning, they have been
employed in a broad variety of problems, such as conditional independence
testing \citep{runge2018conditional,mukherjee2020ccmi}, unsupervised
representation learning \citep{chen2016infogan}, search engine \citep{magerman1990parsing},
and feature selection \citep{peng2005feature}.

The MI of two random variables $X$ and $Y$ measures the expected
point-wise information, where the expectation is taken over the joint
distribution $P_{XY}$. Due to the expectation, estimating mutual
information for continuous variables remains notoriously difficult.
Even if one possesses the specification of the joint distribution,
i.e., a closed form of the density, which is most of the time unknown
in practice, the expectation may still be intractable. Consequently,
exact MI estimation is only possible for discrete random variables.
Historically, MI has been estimated by non-paramtric approachs \citep{kwak2002input,paninski2003estimation,kraskov2004estimating},
which are however not widely applicable due to their unfriendliness
with sample size or dimensionality. Recently, variational approaches
have been proposed to estimate the lower bound of MI \citep{belghazi2018mutual,oord2018representation}.
However, a critical limitation of MI lower bound estimators has been
studied by \citet{mcallester2020formal}, who show that any distribution-free
high-confidence lower bound estimation of mutual information is limited
above by $\mathrm{O}\left(\ln n\right)$ where $n$ is the sample
size. More recent approaches includes hashing \citep{noshad2019scalable},
classifier-based estimator \citep{mukherjee2020ccmi}, and inductive
maximum-entropy copula approach \citep{samo2021inductive}.

While estimating MI is hard, estimating CMI is of magnitudes harder
due to the presence of the conditioning set. Therefore, CMI estimation
methods have seen slower developments than its MI counterparts. Recent
developments for CMI estimation include \citep{runge2018conditional,molavipour2021neural,mukherjee2020ccmi}.

\textbf{Present work.} In this paper, we propose \textbf{DINE}\footnote{Source code and relevant data sets are available at \url{https://github.com/baosws/DINE}.}\textbf{
}(\textbf{D}iffeomorphic \textbf{I}nformation \textbf{N}eural \textbf{E}stimator)--a
unifying framework that closes the gap between the CMI and MI estimation
problems. The approach is advantageous compared with novel variational
methods in the way that it can estimate the exact information measure,
instead of a lower-bound. Specifically, we harness the observation
that CMI is invariant over conditional diffeomorphisms, i.e., differentiable
and invertible maps with differentiable inverse parametrized by the
conditioning variable.

As a direct consequence, first, we can now build a well-designed conditional
diffeomorphic transformation that breaks the statistical dependence
between the conditioning variable with the transformed variables,
but keeps the information measure unchanged, reducing the CMI to an
equivalent MI. Second, the approach offers a complete control over
the distribution form of the newly induced MI estimation problem,
thus we can easily restrict it to an amenable class of simple distributions
with well-established properties and estimate the resultant MI via
available analytic forms. Being aided by the powerful expressivity
of neural networks and normalizing flows \citep{papamakarios2021normalizing},
we can define a rich family of diffeomorphic transformations that
can handle a wide range of non-linear relationships, but are still
efficient in sample size and dimensionality.

\textbf{Contributions.} The key contributions of our study are summarized
as follows:
\begin{itemize}
\item We present a reduction of any CMI estimation problem to an equivalent
MI estimation problem with the unchanged information measure, which
overcomes the central difficulty in CMI estimation compared with MI
estimation.
\item We introduce \textbf{DINE}, a CMI estimator that is flexible, efficient,
and trainable via gradient-based optimizers. We also provide some
theoretical properties of the method.
\item We demonstrate the accuracy of \textbf{DINE} in estimating both MI
and CMI in comparisons with state-of-the-arts under varying sample
sufficiencies, dimensionalities, and non-linear relationships.
\item As a follow-up application of CMI estimation, we also use \textbf{DINE}
to test for conditional independence (CI)--an important statistical
problem with a central role in causality, and show that the test performs
really well, as well as being able to surpass state-of-the-art baselines
by large margins.\selectlanguage{british}%
\end{itemize}

%% file: related.tex
\selectlanguage{english}%
In this Section we formalize the CMI estimation problem and explain
the characterization of CMI that motivated our method.

Regarding notational interpretations, we use capitalized letters $X,Y$,
etc., for random variables/vectors, with lower-case letters $x,y$,
etc., being their respective realizations; the distribution is denoted
by $P_{\cdot}\left(\cdot\right)$ with the respective density $p_{\cdot}\left(\cdot\right)$.

\subsection{Conditional Mutual Information}

The Conditional Mutual Information between continuous random variables
$X$ and $Y$ given $Z$ (with respective compact support sets $\mathcal{X},\mathcal{Y},\mathcal{Z}$)
is defined as

{\small{}
\begin{align}
I\left(X,Y|Z\right) & =\int_{\mathcal{Z}}\int_{\mathcal{Y}}\int_{\mathcal{X}}p\left(x,y,z\right)\ln\frac{p\left(x,y|z\right)}{p\left(x|z\right)p\left(y|z\right)}dxdydz\label{eq:cmi-densities}\\
 & =\mathbb{E}_{p\left(x,y,z\right)}\left[\ln\frac{p\left(x,y|z\right)}{p\left(x|z\right)p\left(y|z\right)}\right]\label{eq:cmi-expectation}
\end{align}
}{\small\par}

where we have assumed that the underlying distributions admit the
corresponding densities $p\left(\cdot\right)$.

Having CMI defined, our technical research question is to estimate
$I\left(X,Y|Z\right)$ using the empirical distribution $P_{XYZ}^{(n)}$
of $n$ i.i.d. samples, without having access to the true distribution
$P_{XYZ}$.

\subsection{Conditional Mutual Information Re-parametrization}

We continue by recalling that MI is invariant via any diffeomorphism
(differentiable bijective transformation with differentiable inverse):
if $x'=\tau_{X}\left(x\right)$ and $y'=\tau_{Y}\left(y\right)$ are
diffeomorphisms, then $I\left(X,Y\right)=I\left(X',Y'\right)$ \citep{kraskov2004estimating}.
Inspired by this attractive property, CMI can be showed to be also
invariant via any \textit{conditional diffeomorphism}, which we define
as
\begin{defn}
(Conditional Diffeomorphism). A differentiable map $\tau\left(\cdot;\cdot\right):\mathcal{X}\times\mathcal{Z}\rightarrow\mathcal{X}'$
is called a conditional diffeomorphism if $\tau\left(\cdot;z\right):\mathcal{X}\rightarrow\mathcal{X}'$
is a diffeomorphism for any $z\in\mathcal{Z}$.
\end{defn}
The following Lemma states that it is possible to re-parametrize CMI
via some conditional diffeomorphisms:
\begin{lem}
\label{lemma:CMI-Re-parametrization}(CMI Re-parametrization). Let
$\tau_{X}:\mathcal{X}\times\mathcal{Z}\rightarrow\mathcal{X}'$ and
$\tau_{Y}:\mathcal{Y}\times\mathcal{Z}\rightarrow\mathcal{Y}'$ be
two conditional diffeomorphisms such that $P_{X'Y'|Z}=P_{X'Y'}$,
where $x'=\tau_{X}\left(x;z\right)$ and $y'=\tau_{Y}\left(y;z\right)$,
then the following holds:

\begin{equation}
I\left(X,Y|Z\right)=I\left(X',Y'\right)\label{eq:equivalence}
\end{equation}

\end{lem}
\begin{proof}
See the Supplementary Material.\selectlanguage{british}%
\end{proof}

%% file: method.tex
\selectlanguage{english}%
Our framework can be described using two main components, namely the
CMI \textit{approximator} and the CMI \textit{estimator}. While the
approximator concerns the hypothesis class of models that are used
to approximate the CMI given the access to the true data distribution,
the CMI estimator defines how to estimate the CMI using models in
the said approximator class, but with only a finite sample size.

\subsection{CMI Approximation}

We start by giving the general CMI approximator based on densities
(as a direct solution to Eqn. \ref{eq:cmi-expectation}):
\begin{defn}
(Density-based CMI approximator). Given a family of density approximators
with parameters $\theta\in\Theta$. The density-based CMI approximator
$I_{\Theta}\left(X,Y|Z\right)$ is defined as

\begin{align}
I_{\Theta}\left(X,Y|Z\right) & =\mathbb{E}_{p\left(x,y,z\right)}\left[\ln\frac{p_{\theta^{*}}\left(x,y|z\right)}{p_{\theta^{*}}\left(x|z\right)p_{\theta^{*}}\left(y|z\right)}\right]\label{eq:DNIM}
\end{align}

where the parameter $\theta^{*}=\left(\theta_{X}^{*},\theta_{Y}^{*},\theta_{XY}^{*}\right)\in\Theta$
are Maximum Likelihood Estimators (MLE) of the true densities $p\left(x,y|z\right)$,
$p\left(x|z\right)$, and $p\left(y|z\right)$:

\begin{align*}
\theta_{X}^{*} & =\argmax_{\theta_{X}}\mathbb{E}_{p\left(x,z\right)}\left[\ln p_{\theta}\left(x|z\right)\right]\\
\theta_{Y}^{*} & =\argmax_{\theta_{Y}}\mathbb{E}_{p\left(y,z\right)}\left[\ln p_{\theta}\left(y|z\right)\right]\\
\theta_{XY}^{*} & =\argmax_{\theta_{XY}}\mathbb{E}_{p\left(x,y,z\right)}\left[\ln p_{\theta}\left(x,y|z\right)\right]
\end{align*}

\end{defn}

The innovation of \textbf{DINE} is fueled by the invariance property
of CMI over diffeomorphic transformations as stated in Eqn.~\ref{eq:equivalence}.
To realize this end, the recently emerging Normalizing Flows (NF)
technique offers us the exact tool we need to exploit the benefits
we have just gained from the CMI re-parametrization.

Simply put, NF offers a general framework to model probability distributions
(in this case $P_{XY|Z}$) by expressing it in terms of a simple ``base''
distribution (here $P_{X'Y'}$) and a series of bijective transformations
(the diffeomorphisms in our method). For more technical details regarding
NFs, see \citep{kobyzev2020normalizing,papamakarios2021normalizing}.

Based on this, our approach involves the design of a class of conditional
normalizing flows (in contrast with the unconditional normalizing
flows that are not parametrized by $Z$), referred to as the \textit{Diffeomorphic
Information Neural Approximator} \textit{(DINA), }and formalized as
\begin{defn}
(Diffeomorphic Information Neural Approximator (DINA)). A DINA $\mathcal{D}_{\Theta}$
is a density-based CMI approximator characterized by the following
elements:
\begin{itemize}
\item A compact parameter domain $\Theta$.
\item A family of base distributions $\left\{ P_{\theta}\left(X',Y'\right)\right\} _{\theta\in\Theta}$.
\item A family of conditional normalizing flows $\left\{ \tau_{\theta}\left(\cdot;\cdot\right):\mathcal{X}\times\mathcal{Z}\rightarrow\mathcal{X}'\right\} _{\theta\in\Theta}$.
\end{itemize}
Then, the approximation is defined as

\[
I_{\Theta}\left(X,Y|Z\right)=I_{\Theta}\left(X',Y'\right)
\]

with $x'=\tau_{\theta_{X}^{*}}\left(x;z\right)$ and $y'=\tau_{\theta_{Y}^{*}}\left(y;z\right)$.
\end{defn}
As mentioned, MI estimation from finite data is still notoriously
difficult if the underlying distribution function is unknown or the
expectation is intractable. Fortunately, the use of normalizing flows
allows us to have a complete control over the distribution of the
surrogate variables $X'$ and $Y'$. Among the choices for the base
distribution, the Gaussian distribution is broadly preferable due
to its well-studied information-theoretic properties, especially the
availability of a closed form MI that we can make us of.

In more details, when the base distribution $P_{\theta}\left(X',Y'\right)$
is jointly Gaussian then we approximate as follows:
\begin{defn}
\label{DINA-for-Gaussian}(DINA-Gaussian). If $P_{\theta}\left(X',Y'\right)$
is multivariate Gaussian, then the DINA approximator with Gaussian
base is defined as

\begin{align*}
I_{\Theta}^{\mathcal{N}}\left(X',Y'\right) & =\frac{1}{2}\ln\frac{\det\Sigma_{p\left(x,z\right)}\left(X'\right)\det\Sigma_{p\left(y,z\right)}\left(Y'\right)}{\det\Sigma_{p\left(x,y,z\right)}\left(X'Y'\right)}
\end{align*}

where $\Sigma_{p\left(x,z\right)}\left(X'\right)$ is the covariance
matrix of $X'$ evaluated on the true distribution $P_{XZ}$, and
so on.
\end{defn}
For the rest of the main Sections we will assume that $P_{\theta}\left(X',Y'\right)$
is jointly multivariate Gaussian with standard Gaussian marginals.
That being said, the framework is still flexible to adapt to arbitrary
base distribution that are independent of $Z$, as long as it is efficient
to evaluate the marginal MI between $X'$ and $Y'$.

We describe in more details the architecture of the normalizing flows
employed for our framework in the next Section.

\subsection{Learning Conditional Diffeomorphisms}

Among a diversely developed literature of normalizing flows \citep{kobyzev2020normalizing,papamakarios2021normalizing},
\textit{autoregressive flows} remain one of the earliest and most
widely adopted. The most attractive characteristic of autoregressive
flows is their intrinsic expressiveness. More concretely, autoregressive
flows are universal approximators of densities \citep{papamakarios2021normalizing},
meaning they can approximate any probability density to an arbitrary
accuracy.

Suppose $X$ and $U$ are $d$-dimensional real-valued random vectors,
where we wish to model the true $p\left(x\right)$ with respect to
the base $p_{\theta}\left(u\right)$. Autoregressive flows transform
each dimension $i$ of $x$ using the information of the dimensions
$1..i-1$ of itself, hence the name ``auto''--regressive:

\[
u_{i}=\tau_{\theta}\left(x_{i};h_{i}\right),\;\text{where }h_{i}=c_{i}\left(x_{<i}\right)
\]

Here the diffeomorphism $\tau$ is referred to as the \textit{transformer},
and the function $c_{i}$ is called the \textit{conditioner}, which
encodes the information from the dimensions $<i$ of $x$ and defines
parameters for the transformer.

Since $u_{i}$ only depends on $x_{\leq i}$, the Jacobian matrix
$J_{\tau}$ of partial derivatives is lower triangular, so the modeled
log density is simply inferred using the change of variables rule
as

\[
\ln p_{\theta}\left(x\right)=\ln p_{\theta}\left(u\right)+\sum_{i=1}^{d}\ln\left|\frac{\partial u_{i}}{\partial x_{i}}\left(u\right)\right|
\]

Furthermore, fitting $p_{\theta}\left(x\right)$ to $p\left(x\right)$
involves maximizing the expected likelihood with respect to the parameter
$\theta$:

\[
\theta^{*}=\argmax_{\theta}\mathbb{E}_{p\left(x\right)}\left[\ln p_{\theta}\left(x\right)\right]
\]

Going back to our problem, for example, to build the conditional diffeomorphism
$x'=\tau_{\theta}\left(x;z\right)$ with a standard Gaussian distribution
$P_{X'}$, we design the base distribution, transformer, and conditioner
as follows:

\subsubsection{Base distribution}

First we choose $P_{X'}$ to be a $d$-dimensional standard uniform
distribution $\mathcal{U}\left(0,1\right)^{d}$, since its density
is constant everywhere, so no computation is required for $\ln p_{\theta}\left(x'\right)$
nor its derivatives.

Next, to transform $P_{X'}$ to the desired standard Gaussian $\mathcal{N}\left(\mathbf{0};\mathbf{1}_{d_{X}}\right)$,
we simply apply the inverse of the cumulative distribution function
(CDF) of the standard Gaussian ($\Phi^{-1}$) to $x'$ in an element-wise
fashion, where $\Phi\left(u\right)=\int_{-\infty}^{u}\mathcal{N}\left(t;0,1\right)dt$.

\subsubsection{Transformer}

With the base distribution being a standard uniform distribution,
a natural choice for the transformer is the CDF of some densities,
which is uniformly distributed for any strictly positive density.

To make the transformation more expressive, we compose the transformer
as a weighted combination of different \textcolor{black}{CDFs} parametrized
by the conditioner. More particularly, the transformer for the $i$-th
dimension is given by

\begin{equation}
\tau\left(x_{i};h_{i}\right)=\sum_{j=1}^{k}w_{ij}\left(h_{i}\right)\Phi\left(x_{i};\mu_{ij}\left(h_{i}\right),\sigma_{ij}^{2}\left(h_{i}\right)\right)\label{eq:gmm-transformer}
\end{equation}

where we have used a mixture of $k$ CDF components, with the $j$-th
component being a Gaussian CDF with mean $\mu_{ij}\left(h_{i}\right)$,
variance $\sigma_{ij}^{2}\left(h_{i}\right)$, and positive weight
$w_{ij}\left(h_{i}\right)$ such that $\sum_{j=1}^{k}w_{ij}\left(\cdot\right)=1$.

This transformer is also a universal approximator for CDFs because
its derivative is essentially a Gaussian mixture model (GMM), a canonical
universal approximator of densities \citep{goodfellow2016deep}. Put
simply, with a sufficient number of Gaussian components, $\tau\left(x_{i};h_{i}\right)$
can express any strictly monotonic $\mathbb{R}\rightarrow\left(0,1\right)$
map (hence invertible) with an arbitrary accuracy, which is followed
by the broad expressiveness of the transformer.

\subsubsection{Conditioner}

Since we would like to model the conditional diffeomorphism $\tau\left(x;z\right)$,
the conditioner function $c_{i}$ must encode both $x_{<i}$ and $z$,
so instead of $h_{i}=c_{i}\left(x_{<i}\right)$, now we let

\begin{equation}
h_{i}=c_{i}\left(x_{<i},z\right)\label{eq:conditioner}
\end{equation}

In contrary to the transformer, the conditioner needs not to be invertible,
so we can freely model it using any family of functions with inputs
$x_{<i}$ and $z$.

\subsubsection{Neural Network Parametrization}

To maximize the expressivity power of autoregressive flows explained
earlier, we parametrize all functional components in Eqn.~\ref{eq:gmm-transformer}
and Eqn.~\ref{eq:conditioner} with neural networks for each dimension,
hence the term ``Neural'' in \textbf{DINE}.

More specifically, let $d_{H}$ be the dimension of $H$, we model
$w_{i}:\mathbb{R}^{d_{H}}\rightarrow\left(0,1\right)^{k}$ as a Multiple
Layer Perceptron (MLP) with Softmax outputs, while $c_{i}:\mathbb{R}^{i-1+d_{Z}}\rightarrow\mathbb{R}^{d_{H}}$,
$\mu_{i}:\mathbb{R}^{d_{H}}\rightarrow\mathbb{R}^{k}$ and $\ln\sigma_{i}^{2}:\mathbb{R}^{d_{H}}\rightarrow\mathbb{R}^{k}$
are real-valued MLPs for all $i=1..d_{X}$.

Since the Jacobian matrix $J_{\tau}$ are now differentiable with
respect to the parameter $\theta$, any gradient-based continuous
optimization framework can be applied to learn $\theta^{*}$.

\subsection{CMI Estimation}

Having the ingredients above ready, we can now proceed to define the
\textbf{DINE} estimator for CMI:
\begin{defn}
(Diffeomorphic Information Neural Estimator (\textbf{DINE})). Consider
a DINA approximator $\mathcal{D}_{\Theta}$ with parameters in a compact
domain $\Theta$. \textbf{DINE} is defined as

\begin{align*}
I_{n}\left(X,Y|Z\right) & =\mathbb{E}_{p^{(n)}\left(x,y,z\right)}\left[\ln\frac{p^{(n)}\left(x',y'\right)}{p^{(n)}\left(x'\right)p^{(n)}\left(y'\right)}\right]
\end{align*}

with $x'=\tau_{\theta_{X}^{*}}\left(x;z\right)$ and $y'=\tau_{\theta_{Y}^{*}}\left(y;z\right)$.

where the parameter $\theta^{*}=\left(\theta_{X}^{*},\theta_{Y}^{*},\theta_{XY}^{*}\right)\in\Theta$
are Maximum Likelihood Estimators (MLE) of the empirical densities
$p^{(n)}\left(x,y|z\right)$, $p^{(n)}\left(x|z\right)$, and $p^{(n)}\left(y|z\right)$:

\begin{align*}
\theta_{X}^{*} & =\argmax_{\theta_{X}}\mathbb{E}_{p^{(n)}\left(x,z\right)}\left[\ln p_{\theta}\left(x|z\right)\right]\\
\theta_{Y}^{*} & =\argmax_{\theta_{Y}}\mathbb{E}_{p^{(n)}\left(y,z\right)}\left[\ln p_{\theta}\left(y|z\right)\right]\\
\theta_{XY}^{*} & =\argmax_{\theta_{XY}}\mathbb{E}_{p^{(n)}\left(x,y,z\right)}\left[\ln p_{\theta}\left(x,y|z\right)\right]
\end{align*}

\end{defn}
Note that $\theta_{X}$ and $\theta_{Y}$ here denote the parameters
of the normalizing flows and the marginal densities $p_{\theta}\left(x'\right)$
and $p_{\theta}\left(y'\right)$, while $\theta_{XY}$ is the parameter
of the joint density $p_{\theta}\left(x',y'\right)$, which is constrained
to have marginals $p_{\theta_{X}^{*}}\left(x'\right)$ and $p_{\theta_{Y}^{*}}\left(y'\right)$.

Specifically, under the case of multivariate Gaussian base, \textbf{DINE}
can be written as
\begin{defn}
\label{DINE-for-Gaussian}(\textbf{DINE-Gaussian}). If $P_{\theta}\left(X',Y'\right)$
is multivariate Gaussian, then the \textbf{DINE} estimator with Gaussian
base is defined as

\[
I_{n}^{\mathcal{N}}\left(X,Y|Z\right)=\frac{1}{2}\ln\frac{\det\Sigma_{n}\left(X'\right)\det\Sigma_{n}\left(Y'\right)}{\det\Sigma_{n}\left(X'Y'\right)}
\]

\end{defn}
\textcolor{blue}{}

This estimator does not require explicit density evaluations but instead
leverages the log determinants of sample covariance matrices, which
offers a low estimation variance. Therefore, from now on, \textbf{DINE}
is assumed to be \textbf{DINE-Gaussian} whenever it is mentioned unless
being explicitly specified.\selectlanguage{british}%

%% file: theory.tex
\selectlanguage{english}%
In this Section we state some important theoretical results regarding
the \textbf{DINE} estimator, including estimation variance, consistency,
and sample complexity.

\subsection{Variance}
\begin{lem}
\label{Variance}(Variance of \textbf{DINE}). The asymptotic variance
of the DINE estimator is given by

\[
\mathrm{Var}\left[I_{n}\right]\stackrel{L}{\rightarrow}\mathrm{O}\left(\frac{d}{n}\right),\;\text{as }n\rightarrow\infty
\]

with $d$ being the dimensionality.

\end{lem}
\begin{proof}
See the Supplementary Material.
\end{proof}

\subsection{Consistency}

The quality of \textbf{DINE} depends on the choice of ($i$) a family
of normalizing flows and ($ii$) $n$ i.i.d. samples from the true
distribution $P_{XYZ}$.

The following Lemma states that, given a sufficiently expressive DINA,
we can approximate the information measure to arbitrary accuracy.
\begin{lem}
\label{Approximation}(Approximability of DINA). For any $\epsilon>0$,
there exists a DINA $\mathcal{D}_{\Theta}$ with some compact domain
$\Theta\subset\mathbb{R}^{c}$ such that

\[
\left|I\left(X,Y|Z\right)-I_{\Theta}\left(X,Y|Z\right)\right|\leq\epsilon,\ \text{almost surely}
\]

\end{lem}
\begin{proof}
See the Supplementary Material.
\end{proof}
The next Lemma declares that the estimator almost surely converges
to the approximator as the sample size approaches infinity.
\begin{lem}
\label{Estimation}(Estimability of \textbf{DINE}). For any $\epsilon>0$,
given a DINA $\mathcal{D}_{\Theta}$ with parameters in some compact
domain $\Theta\subset\mathbb{R}^{c}$, there exists a $N\in\mathbb{N}$
such that

\[
\forall n\geq N,\;\left|I_{n}\left(X,Y|Z\right)-I_{\Theta}\left(X,Y|Z\right)\right|\leq\epsilon,\ \text{almost surely}
\]

\end{lem}
\begin{proof}
See the Supplementary Material.
\end{proof}
Finally, the two Lemmas above together prove the consistency of \textbf{DINE}:
\begin{thm}
\label{Consistency}(Consistency of \textbf{DINE}). \textbf{DINE}
is consistent whenever DINA is sufficiently expressive.
\end{thm}
\begin{proof}
See the Supplementary Material.
\end{proof}

\subsection{Sample Complexity}

Here we state the general sample complexity for general \textbf{DINE},
which is followed by the complexity of \textbf{DINE-Gaussian}.

We make the following assumptions: the log-densities are bounded in
$\left[-M,M\right]$ and $L$-Lipschitz continuous with respect to
the parameters $\theta$, and the parameter domain $\Theta\subset\mathbb{R}^{c}$
is bounded with $\left\Vert \theta\right\Vert \leq K$.
\begin{thm}
\label{Sample-complexity}(Sample complexity of general \textbf{DINE}).
Given any accuracy and confidence parameters $\epsilon,\delta>0$,
the following holds with probability at least $1-\delta$ 

\[
\left|I\left(X,Y|Z\right)-I_{n}\left(X,Y|Z\right)\right|<\epsilon
\]
whenever the sample size $n$ suffices at least

\[
\frac{72M^{2}}{\epsilon^{2}}\left(c\ln\left(\frac{96KL\sqrt{c}}{\epsilon}\right)+\ln\frac{2}{\delta}\right)
\]
\selectlanguage{british}%
\end{thm}

%% file: exp.tex
\selectlanguage{english}%
\begin{figure}[t]
\centering{}\includegraphics[width=1\columnwidth]{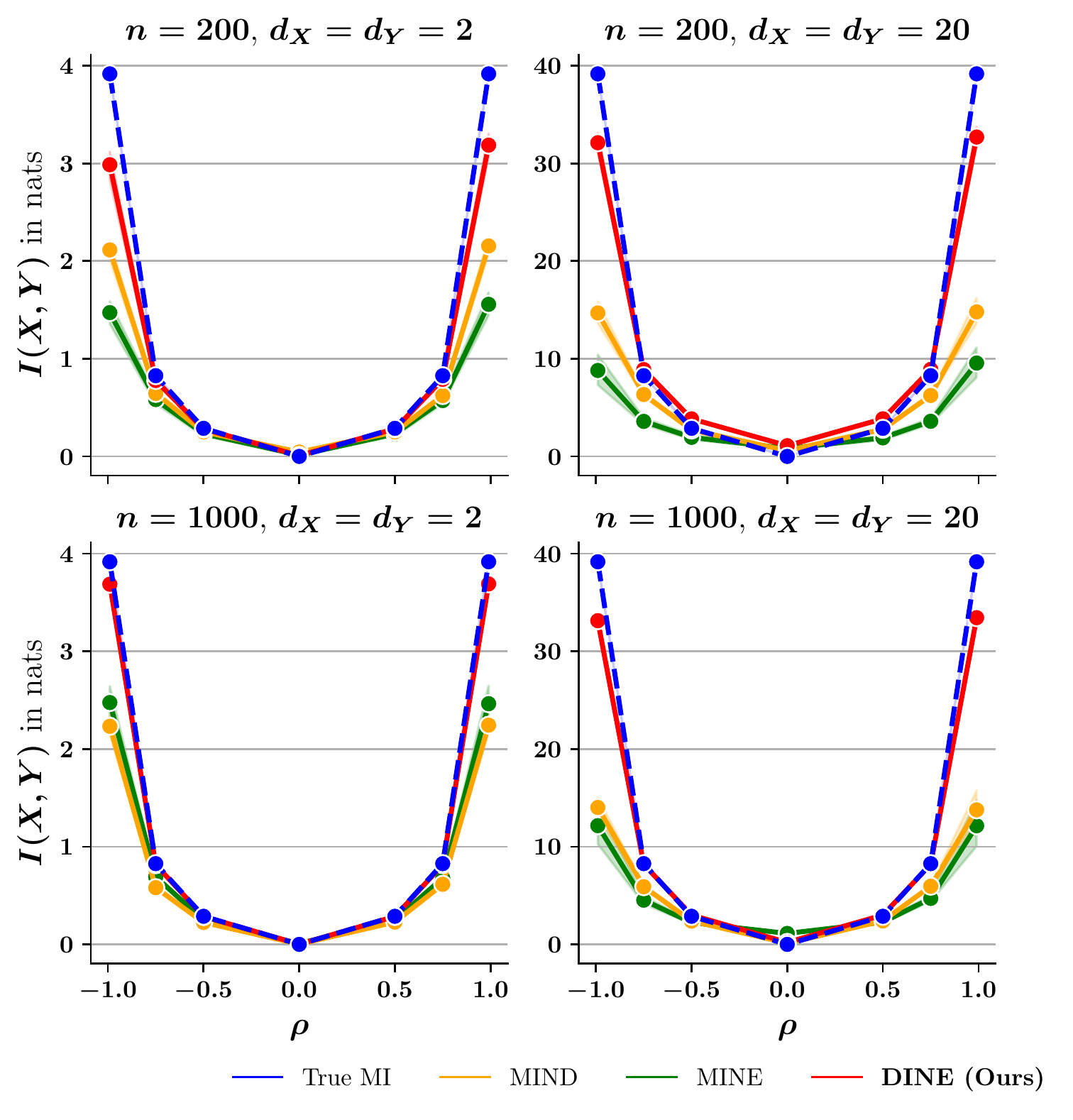}\caption{\label{fig:exp-mi}Mutual Information estimation performance. We compare
the proposed \textbf{DINE} estimator with MINE \citep{belghazi2018mutual}
and MIND \citep{samo2021inductive}. Rows: sample sizes, columns:
dimensionalities. The dashed line denotes the true MI and the other
lines show the averaged estimations for each method over 50 independent
runs. The shaded areas show the estimated 95\% confidence intervals.}
\end{figure}

\begin{figure}[t]
\begin{centering}
\includegraphics[width=1\columnwidth]{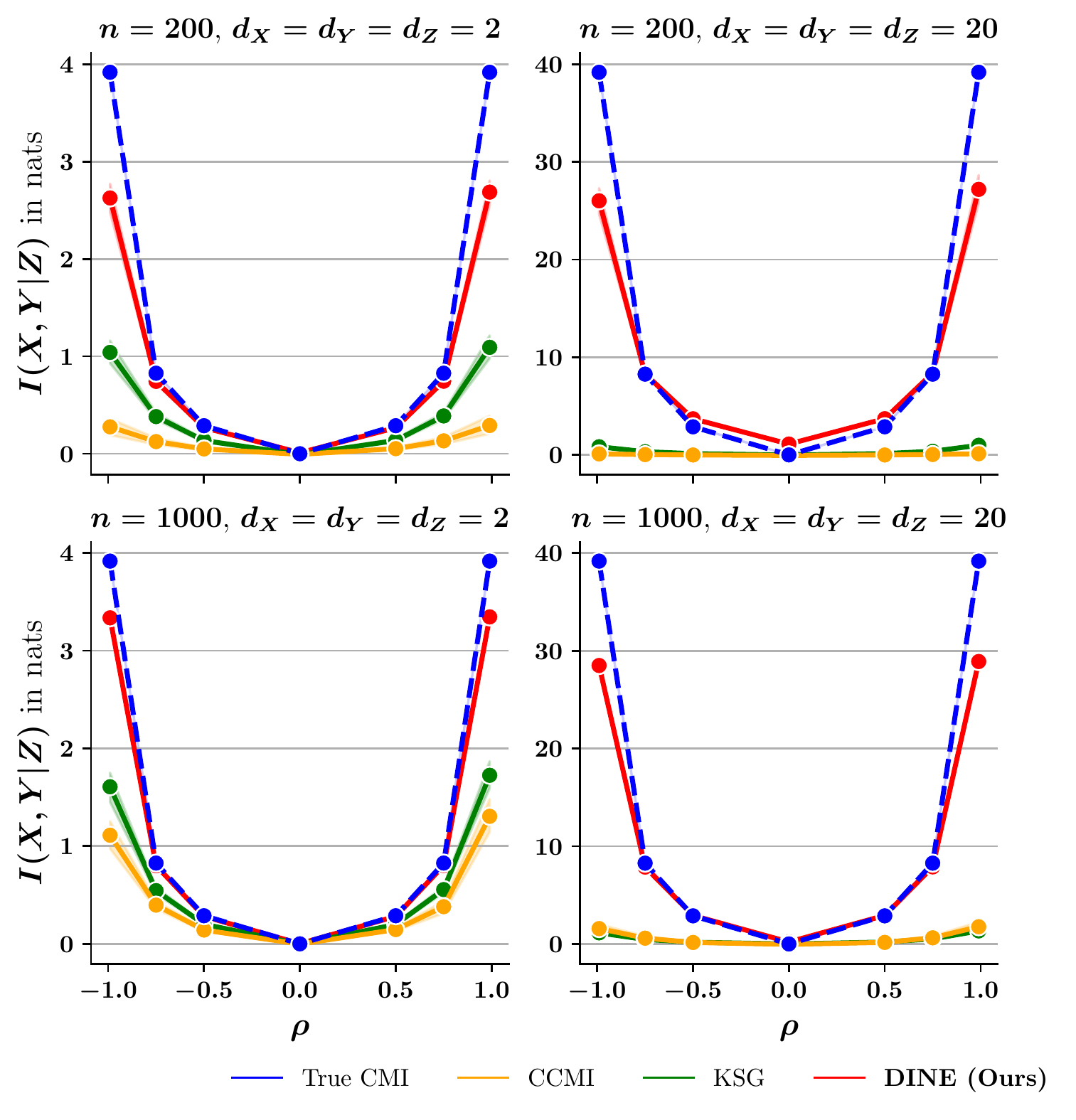}
\par\end{centering}
\caption{\label{fig:exp-cmi}Conditional Mutual Information estimation performance.
We compare the proposed \textbf{DINE} estimator with CCMI \citep{mukherjee2020ccmi}
and KSG \citep{kraskov2004estimating}. Rows: sample sizes, columns:
dimensionalities. The dashed line denotes the true CMI and the other
lines show the averaged estimations for each method over 50 independent
runs. The shaded areas show the estimated 95\% confidence intervals.}
\end{figure}

\begin{figure*}[t]
\centering{}\includegraphics[width=1\textwidth]{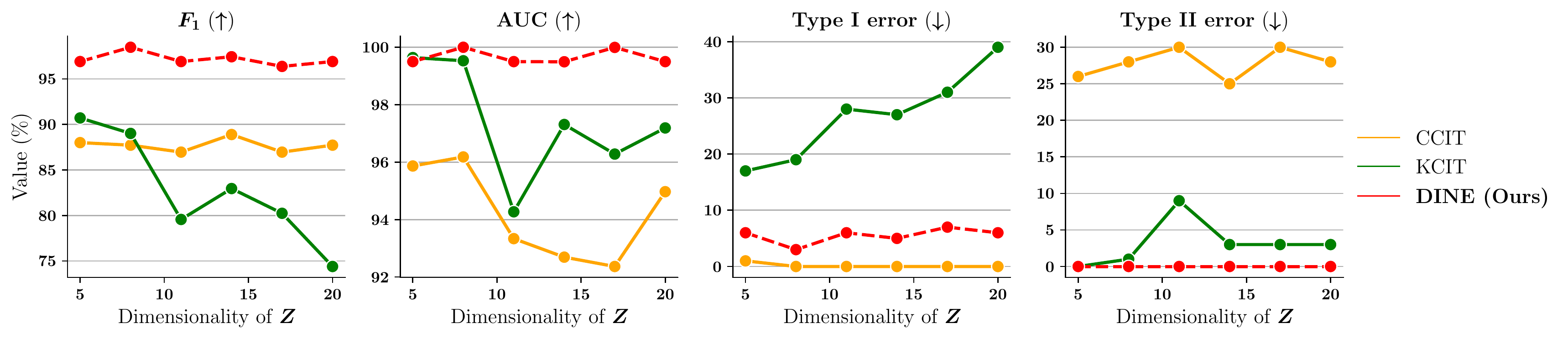}\caption{\label{fig:exp-cit}Conditional Independence Testing performance as
a function of dimensionality. The evaluation metrics are $F_{1}$
score, AUC (higher is better), Type I and Type II error rates (lower
is better), evaluated on 200 independent runs. We compare the proposed
\textbf{DINE}-based CI test with KCIT \citep{zhang2012kernel} and
CCIT \citep{sen2017model}.}
\end{figure*}

In what follows, we illustrate that \textbf{DINE} is far more effective
than the alternative MI and CMI estimators in both sample size and
dimensionality, especially when the actual information measure is
high.

Implementation details and parameters selection of all methods are
given in the Supplementary Material.

\subsection{Synthetic Data}

We consider a diverse set of simulated scenarios covering different
degrees of non-linear dependency, sample size, and dimensionality
settings. For each independent simulation, we first generate two jointly
multivariate Gaussian variables $X',Y'$ with same dimensions $d_{X}=d_{Y}=d$
and shared component-wise correlation, i.e., $\left(X',Y'\right)\sim\mathcal{N}\left(\mathbf{0};\left[\begin{array}{cc}
\mathbf{I}_{d} & \rho\mathbf{I}_{d}\\
\rho\mathbf{I}_{d} & \mathbf{I}_{d}
\end{array}\right]\right)$ with a correlation $\rho\in\left(-1,1\right)$. As for $Z$, we randomly
choose one of three distributions $\mathcal{U}\left(-0.01,0.01\right)^{d_{Z}}$,
$\mathcal{N}\left(\mathbf{0};0.01\mathbf{I}_{d_{Z}}\right)$, and
$\mathrm{Laplace}\left(\mathbf{0};0.01\mathbf{I}_{d_{Z}}\right)$.
Then, $X$ and $Y$ are defined as

\begin{align*}
X & =f\left(AZ+X'\right)\\
Y & =g\left(BZ+Y'\right)
\end{align*}

where $A,B\in\mathbb{R}^{d\times d_{Z}}$ have independent entries
drawn from $\mathcal{N}\left(0;1\right)$, and $f$, $g$ are randomly
chosen from a rich set of mostly non-linear bijective functions $f\left(x\right)\in\left\{ \alpha x,x^{3},e^{-x},\frac{1}{x},\ln x,\frac{1}{1+e^{-x}}\right\} $.\footnote{The input is appropriately scaled and translated before being fed
into the corresponding function.}

By construction, we have the ground truth CMI $I\left(X,Y|Z\right)=I\left(X',Y'\right)=-\frac{d}{2}\ln\left(1-\rho^{2}\right)$.

Finally, $n$ i.i.d. samples $\left\{ \left(x^{(i)},y^{(i)},z^{(i)}\right)\right\} _{i=1}^{n}$
are generated accordingly.

\subsection{Comparison with MI Estimators}

We compare \textbf{DINE} with two state-of-the-arts, the variational
MI lower bound estimator MINE\footnote{We adopt the implementation at \url{https://github.com/karlstratos/doe}}
\citep{belghazi2018mutual} and the inductive copula-based MIND\footnote{We use the author's implementation \url{https://github.com/kxytechnologies/kxy-python/}}
method \citep{samo2021inductive}, which are two of the best approaches
focusing solely on MI estimation.

In this setting, we let $Z$ be empty, i.e., $d_{Z}=0$, and vary
the correlation $\rho$ in\textcolor{red}{{} }$\left[-0.99,0.99\right]$.
We consider both the low sample size $n=200$ and the large sample
size $n=1000$, as well as the low-dimensional $d=2$ and high-dimensional
$d=20$ settings. For each of the setting combinations, we evaluate
the methods using the same 50 independent synthetic data sets according
to the described simulation scheme.

The empirical results are recorded in Figure~\ref{fig:exp-mi}. We
observe that for all scenarios, our \textbf{DINE} method produces
nearly identical estimates with the ground truth, regardless of sample
size or dimensionality, with clear distinctions from MINE and MIND.
Under the most limited setting of low sample size and high dimensionality
(top-right), \textbf{DINE} estimates are still remarkably close the
the ground truth, while MINE and MIND visibly struggles when the ground
truth mutual information is high. On the other hand, for the most
favorable setting of large sample size and low dimensionality (bottom-left),
\textbf{DINE} estimates approach the ground truth with a nearly invisible
margin, whereas the error gaps of MINE and MIND estimates are clearly
distinguishable.

\subsection{Comparison with CMI Estimators}

For this context, we compare \textbf{DINE} with the state-of-the-art
classifier based estimator CCMI\footnote{We use the implementation from the authors at \url{https://github.com/sudiptodip15/CCMI}}
\citep{mukherjee2020ccmi} and the popular \textit{k}-NN based estimator
KSG \citep{kraskov2004estimating}. The experiment setup follows closely
to the MI estimation experiment, except that now we let $d_{Z}=d_{X}=d_{Y}$.

Figure~\ref{fig:exp-cmi} captures the results. We can see that,
compared to the MI estimation setting, the conditioning variable $Z$
degrades the performance of \textbf{DINE}, however only for high ground
truth values and not at a considerable magnitude. Meanwhile, the competitors
CCMI and KSG do not adapt well to the high dimensional setting when
$d_{X}=d_{Y}=d_{Z}=20$. Yet, for the low dimensional case, they still
perform poorly relative to our \textbf{DINE} approach, especially
when the underlying CMI is high.

\selectlanguage{british}%

%% file: app.tex
\selectlanguage{english}%
Among the broad range of applications of CMI estimation, the Conditional
Independence (CI) test is perhaps one of the most desired. CI testing
greatly benefits the field of Causal Discovery \citep{spirtes2000causation}.
Therefore, in this Section we illustrate that our approach may be
used to construct a CI test that strongly outperforms other competitive
baselines solely designed for the same goal, as a down-stream evaluation
of \textbf{DINE}. Resultantly, the test can be expected to improve
Causal Discovery methods significantly.

\subsection{Context}

Formally, CI testing concerns with the statistical hypothesis test
with

\begin{align*}
\mathcal{H}_{0} & :X\indep Y\mid Z\\
\mathcal{H}_{1} & :X\not\indep Y\mid Z
\end{align*}

\subsubsection{Related works}

In the context of CI testing, kernel-based approaches \citep{zhang2012kernel},
are generally the most popular and powerful methods for CI testing,
which adopt kernels to exploit high order statistics that capture
the CI structure of the data. Recently, more modern approaches have
also been proposed, such as GAN-based \citep{shi2021double} or classification-based
\citep{sen2017model} methods with promising results.

\subsubsection{Description of the test}

We design a simple \textbf{DINE}-based CI test inspired by the observation
that $X\indep Y\mid Z\Leftrightarrow I\left(X',Y'\right)=0$ and use
$I\left(X',Y'\right)$ as the test statistics. Next, we employ the
permutation technique \citep{doran2014permutation} to simulate the
null distribution of the test statistics and estimate the $p$-value.
Finally, given a user-defined significance level $\alpha$, we reject
the null hypothesis $\mathcal{H}_{0}$ if $p\text{-value}<\alpha$
and accept it otherwise.

\subsection{Experiments}

To numerically evaluate the quality of the aforementioned \textbf{DINE}-based
CI test, we compare it with the prominent kernel-based test KCIT\footnote{The implementation from the causal-learn package is adopted \url{https://github.com/cmu-phil/causal-learn}}
\citep{zhang2012kernel} and a more recent state-of-the-art classifier-based
test CCIT\footnote{The authors' implementation can be found at \url{https://github.com/rajatsen91/CCIT}}
\citep{sen2017model}.

In this experiment, we fix $d_{X}=d_{Y}=1$ and consider $d_{Z}$
increasing from low to high dimensionalities in $\left[5,20\right]$,
with a constant sample size $n=1000$, and compare the performance
of \textbf{DINE} against the baselines, assessed under four different
criterions, namely the $F_{1}$ score, AUC (higher is better), Type
I, and Type II error rates (lower is better). These metrics are evaluated
using 200 independent runs (100 runs for each label) for each combination
of method and dimensionality of $Z$. Additionally, for $F_{1}$ score,
Type I, and Type II errors, we adopt the common significance level
of $\alpha=0.05$. Furthermore, for the case of conditional independence
we let $\rho=0$, whereas for the conditional dependence case we randomly
draw $\rho\sim\mathcal{U}\left(\left[-0.99,0.1\right]\cup\left[0.1,0.99\right]\right)$.
The data is generated according to the CMI experiment in the previous
Section.

The numerical comparisons in CI testing are presented in Figure~\ref{fig:exp-cit},
which show that the \textbf{DINE}-based CI test obtains very good
scores under all performance metrics. More particularly, it nearly
never makes any Type II error, meaning when the relationship is actually
conditional dependence, the CMI estimate is rarely too low to be misclassified
as conditional independence; meanwhile, its Type I errors are roughly
proximate to the rejection threshold $\alpha$, which is expected
from the definition of $p$-value. Moreover, the $F_{1}$ and AUC
scores of \textbf{DINE} are also highest in all cases and closely
approach 100\%, suggesting the superior adaptability of \textbf{DINE}
to both non-linear relationships and higher dimensionalities.

Regarding the baseline methods, KCIT and CCIT show completely opposite
behaviors to each other. While KCIT has relatively low Type II errors,
its Type I errors are quite high even at lower dimensionalities and
increase rapidly in the increment dimensionality. Conversely, CCIT
is quite conservative in Type I error as a trade-off for the consistently
high Type II errors. However, their AUC scores are still high, indicating
that the optimal threshold exists, but their $p$-value estimates
do not reflect accurately the true $p$-value.\selectlanguage{british}%

%% file: conclude.tex
In this paper we propose \textbf{DINE}, a novel approach for CMI estimation.
Through the use of normalizing flows, we simplify the challenging
CMI estimation problem into the easier MI estimation, which can be
designed to be efficiently evaluable, overcoming the inherent difficulties
in existing approaches. We compare \textbf{DINE} with best-in-class
methods for MI estimation, CMI estimation, in which \textbf{DINE}
shows considerably better performance as compared to its counterparts,
as well as being friendly in sample size and dimensionality while
adapting well to several non-linear relationships. Finally, we show
that \textbf{DINE} can also be used to define a CI test with an improved
effectiveness in comparison with state-of-the-art CI tests, thanks
to its accurate CMI estimability.

%% file: Supplementary.tex
\selectlanguage{english}%

\section*{Supplementary Material for \textquotedblleft Diffeomorphic Information
Neural Estimation\textquotedblright{}}

\RestyleAlgo{ruled}

\begin{algorithm}[th]
\caption{The \textbf{DINE-Gaussian} Algorithm for CMI estimation.\label{alg:DINE-Gaussian}}

\textbf{Input:} Empirical samples $P_{XYZ}^{(n)}=\left\{ \left(x_{i},y_{i},z_{i}\right)\right\} _{i=1}^{n}$.

\textbf{Output:} An estimation of $I\left(X,Y|Z\right)$.
\begin{enumerate}
\item Define two autoregressive flows $\tau_{\theta_{X}}\left(x;z\right)$
and $\tau_{\theta_{X}}\left(y;z\right)$ with the standard Gaussian
base densities $p_{\theta}\left(x'\right)$, $p_{\theta}\left(y'\right)$.
\item Learn $\theta^{*}$ by maximum likelihood estimation:

\[
\theta^{*}=\argmax_{\theta}\frac{1}{n}\left[\ln p_{\theta}\left(x_{i}|z_{i}\right)+\ln p_{\theta}\left(y_{i}|z_{i}\right)\right]
\]

where

{\footnotesize{}
\begin{align*}
\ln p_{\theta}\left(x_{i}|z_{i}\right) & =\ln p_{\theta}\left(\tau_{\theta_{X}}\left(x_{i};z_{i}\right)\right)+\ln\left|\frac{\partial\tau_{\theta_{X}}}{\partial x}\left(x_{i};z_{i}\right)\right|\\
\ln p_{\theta}\left(y_{i}|z_{i}\right) & =\ln p_{\theta}\left(\tau_{\theta_{Y}}\left(y_{i};z_{i}\right)\right)+\ln\left|\frac{\partial\tau_{\theta_{Y}}}{\partial y}\left(y_{i};z_{i}\right)\right|
\end{align*}
}{\footnotesize\par}
\item Evaluate $x'_{i}$ and $y_{i}'$ using the learned parameter $\theta^{*}$
for all $i=1..n$:

\begin{align*}
x_{i}' & =\tau_{\theta_{X}^{*}}\left(x_{i};z_{i}\right)\\
y_{i}' & =\tau_{\theta_{Y}^{*}}\left(y_{i};z_{i}\right)
\end{align*}

\item Estimate the (uncentered) sample covariance matrices:

\begin{align*}
\Sigma_{X'}^{(n)} & =\frac{1}{n-1}\sum_{i=1}^{n}x'_{i}x{}_{i}^{\prime T}\\
\Sigma_{Y'}^{(n)} & =\frac{1}{n-1}\sum_{i=1}^{n}y'_{i}y{}_{i}^{\prime T}\\
\Sigma_{X'Y'}^{(n)} & =\frac{1}{n-1}\sum_{i=1}^{n}u'_{i}u{}_{i}^{\prime T}
\end{align*}

where $u_{i}$ is the concatenation of $x'_{i}$ and $y'_{i}$.
\item Return the estimate:

\[
\frac{1}{2}\ln\frac{\det\Sigma_{X'}^{(n)}\det\Sigma_{Y'}^{(n)}}{\det\Sigma_{X'Y'}^{(n)}}
\]
\end{enumerate}
\end{algorithm}

\RestyleAlgo{ruled}

\begin{algorithm}[th]
\caption{The \textbf{DINE}-based\textbf{ }Conditional Independence Testing
Algorithm.\label{alg:DINE-CI-test}}

\textbf{Input:} Empirical samples $P_{XYZ}^{(n)}=\left\{ \left(x_{i},y_{i},z_{i}\right)\right\} _{i=1}^{n}$,
number of bootstraps $B$, and a significance level $\alpha$.

\textbf{Output:} The test's $p$-value and whether $X\indep Y|Z$
or not.
\begin{enumerate}
\item Use the Algorithm~\ref{alg:DINE-Gaussian} to infer $\left\{ \left(x'_{i},y_{i}'\right)\right\} _{i=1}^{n}$
and the CMI measure, denoted as $I$.
\item For each $i=1..B$:
\begin{enumerate}
\item Randomly permutate $\left\{ \left(y'_{i}\right)\right\} _{i=1}^{n}$
while keeping $\left\{ \left(x'_{i}\right)\right\} _{i=1}^{n}$ unchanged.
\item Repeat steps 4 and 5 of Algorithm~\ref{alg:DINE-Gaussian} to estimate
the CMI of the permutated samples, denoted as $I_{i}$.
\end{enumerate}
\item Calculate the $p$-value:

\[
p\text{-value}=\frac{1}{B}\sum_{i=1}^{B}\mathbf{1}\left\{ I\leq I_{i}\right\} 
\]

\item Return $p$-value and

\[
\text{Conclusion}:=\begin{cases}
X\indep Y|Z & \text{if }p\text{-value}>\alpha\\
X\not\indep Y|Z & \text{if }p\text{-value \ensuremath{\leq\alpha}}
\end{cases}
\]
\end{enumerate}
\end{algorithm}

In this Supplementary Material we provide the followings:
\begin{itemize}
\item Description of additional experiment settings for all methods.
\item Proofs of formal claims.
\item Implementation details of proposed algorithms, including \textbf{DINE-Gaussian}
and \textbf{DINE}-based CI test.
\end{itemize}
\selectlanguage{british}%

\selectlanguage{english}%

\subsection{Additional experiment settings}

For all baseline methods, we use the default parameters recommended
in the respective open-source implementation.

As for \textbf{DINE}, the Multiple Layer Perceptrons (MLP) employed
each has one hidden layer with ReLU activation function. The most
important two hyper-parameters in our implementation are the number
of units in the hidden layer and the number of components for the
Gaussian Mixtures Model. Intuitively, the higher these values, the
more expressive \textbf{DINE} becomes. However, as the empirical evaluation
suggested, a very light-weight model with only four hidden units per
MLP and 16 Gaussian mixture components (that we use for all experiments)
suffices to outperform all state-of-the-art methods. Meanwhile, MINE,
for example, uses an MLP with 128 hidden units (overfitting is not
possible for this approach), but still cannot perform comparably to
our method.

More interestingly, in comparison with most of other methods, our
approach supports hyper-parameter tuning with the cross-validation
maximum likelihood as the objective. The only baseline method that
also supports auto hyper-parameter selection is CCIT, yet with the
hyper-parameters thoroughly searched as suggested by the author's
implementation, it is still surpassed by \textbf{DINE}.

\subsection{Proofs}

\textbf{Assumptions. }We employ the following assumptions throughout
the proofs:
\begin{enumerate}
\item The underlying data distribution $P_{XYZ}$ admits a density function
$p\left(x,y,z\right)$ in a compact domain $\mathcal{X}\times\mathcal{Y}\times\mathcal{Z}\subset\mathbb{R}^{d_{X}\times d_{Y}\times d_{Z}}$.
\item The parameter domain $\Theta\subset\mathbb{R}^{c}$ is bounded with
$\left\Vert \theta\right\Vert \leq K$ with some constant $K>0$.
\item The modeling densities $p_{\theta}\left(\cdot\right)$ are bounded
in $\left[e^{-M},e^{M}\right]$ with constant $M>0$, and $L$-Lipschitz
with respect to the parameter $\theta$.
\end{enumerate}

\paragraph{Proof of CMI invariance}
\begin{lem}
(CMI Re-parametrization, Lemma 1 in the main text). Let $\tau_{X}:\mathcal{X}\times\mathcal{Z}\rightarrow\mathcal{X}'$
and $\tau_{Y}:\mathcal{Y}\times\mathcal{Z}\rightarrow\mathcal{Y}'$
be two conditional diffeomorphisms such that $P_{X'Y'|Z}=P_{X'Y'}$,
where $x'=\tau_{X}\left(x;z\right)$ and $y'=\tau_{Y}\left(y;z\right)$,
then the following holds:

\[
I\left(X,Y|Z\right)=I\left(X',Y'\right)
\]

\end{lem}
\begin{proof}
By writing the CMI as the expected $z$-specific MI, and using the
fact that it is invariant via the diffeomorphisms $x'=\tau_{X}\left(x;z\right)$
and $y'=\tau_{Y}\left(y';z\right)$, we have:

\begin{align}
I\left(X,Y|Z\right) & =\mathbb{\mathbb{E}}_{p\left(z\right)}\left[I\left(X,Y|Z=z\right)\right]\\
 & =\mathbb{E}_{p\left(z\right)}\left[I\left(X',Y'|Z=z\right)\right]\\
 & =\mathbb{E}_{p\left(z\right)}\left[I\left(X',Y'\right)\right]\\
 & =I\left(X',Y'\right)
\end{align}

where the third and forth equalities are due to the fact that $I\left(X',Y'\right)$
is a function of $P_{X'Y'}$, which does not depend on $z$ by our
constraint.
\end{proof}

\paragraph{Proof of Variance}
\begin{lem}
(Variance of \textbf{DINE-Gaussian}, Lemma 2 in the main text). The
asymptotic variance of the DINE estimator is given by

\[
\mathrm{Var}\left[I_{n}\right]\stackrel{L}{\rightarrow}\mathrm{O}\left(\frac{d}{n}\right)=\mathrm{O}\left(\frac{1}{n}\right)
\]

with $d$ being the dimensionality.

\end{lem}
\begin{proof}
The estimation variance of DINE-Gaussian is bounded by

\begin{align}
\mathrm{Var}\left[I_{n}\right] & \leq\frac{1}{3}(\mathrm{Var}\left[\ln\det\Sigma_{X'}^{(n)}\right]\\
 & +\mathrm{Var}\left[\ln\det\Sigma_{Y'}^{(n)}\right]\\
 & +\mathrm{Var}\left[\ln\det\Sigma_{X'Y'}^{(n)}\right])
\end{align}

Let $\Sigma=\mathbb{E}\left[\Sigma^{(n)}\right]$, then according
to the limiting law of the log determinant of the sample covariance
matrix \citep[Corollary 1,][]{cai2015law}, \textit{$\ln\det\Sigma^{(n)}-\ln\det\Sigma$}
has an asymptotic distribution $\mathcal{N}\left(\frac{d\left(d+1\right)}{2n},\frac{2d}{n}\right)$
as $n\rightarrow\infty$, where $d$ is the dimensionality of $\Sigma$.
Thus the asymptotic variance of $\ln\det\Sigma^{(n)}$ is given by

\begin{align}
\mathrm{Var}\left[\ln\det\Sigma^{(n)}\right] & \stackrel{L}{\rightarrow}\frac{2d}{n}+\mathrm{O}\left(\left(\frac{d^{2}}{n}\right)^{2}\right)\\
 & =\mathrm{O}\left(\frac{d}{n}\right)
\end{align}

and subsequently the variance of the \textbf{DINE} estimator is

\begin{equation}
\mathrm{Var}\left[I_{n}\right]\stackrel{L}{\rightarrow}\mathrm{O}\left(\frac{d}{n}\right)=\mathrm{O}\left(\frac{1}{n}\right)
\end{equation}

\end{proof}

\paragraph{Proofs of Consistency}
\begin{lem}
(Approximability, Lemma 3 in the main text). For any $\epsilon>0$,
there exists a DINA $\mathcal{D}_{\Theta}$ with some compact parameter
domain $\Theta\subset\mathbb{R}^{c}$ such that

\[
\left|I\left(X,Y|Z\right)-I_{\Theta}\left(X,Y|Z\right)\right|\leq\epsilon,\ \text{almost surely}
\]

\end{lem}
\begin{proof}
We first recall by the definition that

\begin{align}
I_{\Theta}\left(X,Y|Z\right) & =\mathbb{E}_{p\left(x,y,z\right)}\left[\ln\frac{p_{\theta^{*}}\left(x,y|z\right)}{p_{\theta^{*}}\left(x|z\right)p_{\theta^{*}}\left(y|z\right)}\right]
\end{align}

Using the triangular inequality, we can bound the approximation gap
as

\begin{align}
 & \left|I\left(X,Y|Z\right)-I_{\Theta}\left(X,Y|Z\right)\right|\\
\leq & \left|\mathbb{E}_{p\left(x,y,z\right)}\left[\ln p\left(x,y|z\right)-\ln p_{\theta^{*}}\left(x,y|z\right)\right]\right|\\
+ & \left|\mathbb{E}_{p\left(x,z\right)}\left[\ln p\left(x|z\right)-\ln p_{\theta^{*}}\left(x|z\right)\right]\right|\\
+ & \left|\mathbb{E}_{p\left(y,z\right)}\left[\ln p\left(y|z\right)-\ln p_{\theta^{*}}\left(y|z\right)\right]\right|
\end{align}

Assuming the learnable parameters $\theta_{X},\theta_{Y}$, and $\theta_{XY}$
are disjointed, so that their expected maximum likelihoods are unconstrained
of each other. For example, we let $\theta_{X}$ and $\theta_{Y}$
be the autoregressive flows' parameters, while $\theta_{XY}$ defines
the base distribution. Now we consider the first term of the right
hand side above, in which the argument is followed similarly by the
last two terms.

Fix $\epsilon>0$. By the universal approximability of autoregressive
flows, $\ln p_{\theta}\left(\cdot\right)$ can approximate any log
density function with arbitrary accuracy. More specifically, there
exists some parameter domain $\Theta$ that defines the approximator
$\mathcal{D}_{\Theta}$ such that,

\begin{align}
\mathbb{E}_{p\left(x,y,z\right)}\left|\ln p\left(x,y|z\right)-\ln p_{\theta}\left(x,y|z\right)\right| & \leq\frac{\epsilon}{3}\label{eq:universal-approx}
\end{align}

and by the fact that $\theta^{*}$ has the maximum expected likelihood
by design:

\begin{align}
 & \mathbb{E}_{p\left(x,y,z\right)}\left[\ln p\left(x,y|z\right)-\ln p_{\theta^{*}}\left(x,y|z\right)\right]\\
\leq\  & \mathbb{E}_{p\left(x,y,z\right)}\left[\ln p\left(x,y|z\right)-\ln p_{\theta}\left(x,y|z\right)\right]\\
\leq\  & \mathbb{E}_{p\left(x,y,z\right)}\left|\ln p\left(x,y|z\right)-\ln p_{\theta}\left(x,y|z\right)\right|\\
\leq\  & \frac{\epsilon}{3}
\end{align}

where the second inequality follows the triangular inequality.

Combining the above with the same arguments for $p_{\theta}\left(x|z\right)$
and $p_{\theta}\left(y|z\right)$ completes the proof:

\begin{align}
\left|I\left(X,Y|Z\right)-I_{\Theta}\left(X,Y|Z\right)\right| & \leq\frac{\epsilon}{3}+\frac{\epsilon}{3}+\frac{\epsilon}{3}=\epsilon
\end{align}

\end{proof}
\begin{lem}
(Estimability, Lemma 4 in the main text). For any $\epsilon>0$, given
a DINA $\mathcal{D}_{\Theta}$ with parameters in some compact domain
$\Theta\subset\mathbb{R}^{c}$, there exists a $N\in\mathbb{N}$ such
that

\[
\forall n\geq N,\;\left|I_{n}\left(X,Y|Z\right)-I_{\Theta}\left(X,Y|Z\right)\right|\leq\epsilon,\ \text{almost surely}
\]

\end{lem}
\begin{proof}
Using the triangular inequality we have

{\footnotesize{}
\begin{align}
 & \left|I_{n}\left(X,Y|Z\right)-I_{\Theta}\left(X,Y|Z\right)\right|\\
\leq & \left|\mathbb{E}_{p^{(n)}\left(x,y,z\right)}\left[\ln p_{\theta_{n}^{*}}\left(x,y|z\right)\right]-\mathbb{E}_{p\left(x,y,z\right)}\left[\ln p_{\theta^{*}}\left(x,y|z\right)\right]\right|\\
+ & \left|\mathbb{E}_{p^{(n)}\left(x,z\right)}\left[\ln p_{\theta_{n}^{*}}\left(x|z\right)\right]-\mathbb{E}_{p\left(x,z\right)}\left[\ln p_{\theta^{*}}\left(x|z\right)\right]\right|\\
+ & \left|\mathbb{E}_{p^{(n)}\left(y,z\right)}\left[\ln p_{\theta_{n}^{*}}\left(y|z\right)\right]-\mathbb{E}_{p\left(y,z\right)}\left[\ln p_{\theta^{*}}\left(y|z\right)\right]\right|\label{eq:est-tri}
\end{align}
}{\footnotesize\par}

Consider any term, we rewrite

{\small{}
\begin{align}
 & \left|\mathbb{E}_{p^{(n)}}\left[\ln p_{\theta_{n}^{*}}\left(\cdot\right)\right]-\mathbb{E}_{p}\left[\ln p_{\theta^{*}}\left(\cdot\right)\right]\right|\\
= & \left|\max_{\theta\in\Theta}\mathbb{E}_{p^{(n)}}\left[\ln p_{\theta}\left(\cdot\right)\right]-\max_{\theta\in\Theta}\mathbb{E}_{p}\left[\ln p_{\theta}\left(\cdot\right)\right]\right|\\
\leq & \max_{\theta\in\Theta}\left|\mathbb{E}_{p^{(n)}}\left[\ln p_{\theta}\left(\cdot\right)\right]-\mathbb{E}_{p}\left[\ln p_{\theta}\left(\cdot\right)\right]\right|\label{eq:large-number}
\end{align}
}{\small\par}

Since $\ln p_{\theta}\left(\cdot\right)$ is continuous with finite
expectation and variance, the law of large numbers entails that given
any $\epsilon>0$, we can choose a $N\in\mathbb{N}$ such that almost
surely $\forall n\geq N$:

\begin{equation}
\left|\mathbb{E}_{p^{(n)}}\left[\ln p_{\theta}\left(\cdot\right)\right]-\mathbb{E}_{p}\left[\ln p_{\theta}\left(\cdot\right)\right]\right|\leq\frac{\epsilon}{3}
\end{equation}

The same arguments can be applied with the remaining two terms in
Eqn.~\ref{eq:est-tri}. Finally, together they give

\begin{equation}
\left|I_{n}\left(X,Y|Z\right)-I_{\Theta}\left(X,Y|Z\right)\right|\leq\frac{\epsilon}{3}+\frac{\epsilon}{3}+\frac{\epsilon}{3}=\epsilon
\end{equation}

\end{proof}
\begin{thm}
(Consistency, Theorem 1 in the main text). Given $\epsilon>0$, there
exists a parameter domain $\Theta$ and a $N\in\mathbb{N}$ such that
$\forall n\geq N$,

\[
\left|I_{n}\left(X,Y|Z\right)-I\left(X,Y|Z\right)\right|\leq\epsilon,\;\text{almost surely}
\]

\end{thm}
\begin{proof}
For any $\epsilon>0$, by Lemma 3 we can choose a DINA $\mathcal{D}_{\Theta}$
with parameter domain $\Theta$ such that $\left|I\left(X,Y|Z\right)-I_{\Theta}\left(X,Y|Z\right)\right|\leq\frac{\epsilon}{2}$,
and by Lemma 4 we can choose a $N\in\mathbb{N}$ such that $\forall n\geq N$,
$\left|I_{n}\left(X,Y|Z\right)-I_{\Theta}\left(X,Y|Z\right)\right|\leq\frac{\epsilon}{2}$. 

Combining these two with the triangular inequality, we have with probability
one that $\forall n\geq N$:

\begin{align}
 & \left|I_{n}\left(X,Y|Z\right)-I\left(X,Y|Z\right)\right|\\
\leq & \left|I_{n}\left(X,Y|Z\right)-I_{\Theta}\left(X,Y|Z\right)\right|\\
+ & \left|I\left(X,Y|Z\right)-I_{\Theta}\left(X,Y|Z\right)\right|\\
\leq & \frac{\epsilon}{2}+\frac{\epsilon}{2}\\
= & \ \epsilon
\end{align}

\end{proof}
\begin{thm}
(Sample complexity, Theorem 2 in the main text). Given any accuracy
and confidence parameters $\epsilon,\delta>0$, the following holds
with probability at least $1-\delta$ 

\[
\left|I_{n}\left(X,Y|Z\right)-I\left(X,Y|Z\right)\right|\leq\epsilon
\]

\end{thm}
\begin{proof}
The argument is similar to Theorem 1, with more attention to Lemma
4.

More specifically, based on Lemma 3, first we choose a domain $\Theta$
such that $\left|I\left(X,Y|Z\right)-I_{\Theta}\left(X,Y|Z\right)\right|\leq\frac{\epsilon}{2}$.

Next, using the Hoeffing inequality for the bounded function $\ln p_{\theta}\left(\cdot\right)\in\left[-M,M\right]$
yields

{\large{}
\begin{align}
 & p\left(\left|\mathbb{E}_{p^{(n)}}\left[\ln p_{\theta}\left(\cdot\right)\right]-\mathbb{E}_{p}\left[\ln p_{\theta}\left(\cdot\right)\right]\right|>\frac{\epsilon}{6}\right)\\
 & \leq2\exp\left(-\frac{\epsilon^{2}n}{72M^{2}}\right)
\end{align}
}{\large\par}

To make this inequality satisfied uniformly for all parameters $\theta\in\Theta\subset\mathbb{R}^{c}$,
we employ the common technique of choosing a minimal cover set of
the domain $\Theta$ by a finite set of hypersphere $\left\{ \mathcal{B}_{r}\left(\theta_{j}\right)\right\} _{j=1}^{m}$
of radius $r$, such that $\Theta\subset\Theta_{r}=\bigcup_{j=1}^{m}\mathcal{B}_{r}\left(\theta_{j}\right)$,
and apply the union bound inequality. The cardinality $m$ of the
minimal cover set is finite and bounded by the \textit{covering number}
$N_{r}\left(\Theta\right)$, where:
\begin{equation}
N_{r}\left(\Theta\right)\leq\left(\frac{2K\sqrt{c}}{r}\right)^{c}
\end{equation}

Using the union bound on $\Theta_{r}$ gives

{\small{}
\begin{align}
 & p\left(\max_{\theta\in\Theta}\left|\mathbb{E}_{p^{(n)}}\left[\ln p_{\theta}\left(\cdot\right)\right]-\mathbb{E}_{p}\left[\ln p_{\theta}\left(\cdot\right)\right]\right|>\frac{\epsilon}{6}\right)\\
\leq & \sum_{j=1}^{m}p\left(\left|\mathbb{E}_{p^{(n)}}\left[\ln p_{\theta_{j}}\left(\cdot\right)\right]-\mathbb{E}_{p}\left[\ln p_{\theta_{j}}\left(\cdot\right)\right]\right|>\frac{\epsilon}{6}\right)\\
\leq & 2N_{r}\left(\Theta\right)\exp\left(-\frac{\epsilon^{2}n}{72M^{2}}\right)\\
\leq & 2\left(\frac{2K\sqrt{c}}{r}\right)^{c}\exp\left(-\frac{\epsilon^{2}n}{72M^{2}}\right)
\end{align}
}{\small\par}

Now consider $r=\frac{\epsilon}{48L}$ \citep{mohri2018foundations},
let the bound above be at most $\delta$ and solve for $n$, we obtain

\begin{equation}
n\geq\frac{72M^{2}}{\epsilon^{2}}\left(c\ln\frac{96KL\sqrt{c}}{\epsilon}+\ln\frac{2}{\delta}\right)
\end{equation}

Using the same arguments in Theorem 1, we conclude that if $n$ satisfies
the sample size above then with probability at least $1-\delta$,

\begin{align}
 & \left|I_{n}\left(X,Y|Z\right)-I\left(X,Y|Z\right)\right|\\
\leq & \left|I_{n}\left(X,Y|Z\right)-I_{\Theta}\left(X,Y|Z\right)\right|\\
+ & \left|I\left(X,Y|Z\right)-I_{\Theta}\left(X,Y|Z\right)\right|\\
\leq & \frac{\epsilon}{2}+\frac{\epsilon}{6}+\frac{\epsilon}{6}+\frac{\epsilon}{6}\\
= & \ \epsilon
\end{align}

which proves the polynomial convergence rate (in $\frac{1}{\epsilon}$
and $\frac{1}{\delta}$) of general \textbf{DINE}.
\end{proof}

\subparagraph{}

\subsection{Implementation details}

First, let us note that the \textbf{DINE-Gaussian} estimator does
not require learning the parameter $\theta_{XY}^{*}$ as it is not
involved in the formulation, therefore we only need to optimize for
$\theta_{X}^{*}$ and $\theta_{Y}^{*}$, which are disjointed so they
can be learned using a single combined objective function. See Algorithm~\ref{alg:DINE-Gaussian}
for the highlighted main steps of \textbf{DINE-Gaussian}. In addition,
Algorithm~\ref{alg:DINE-CI-test} summarizes the \textbf{DINE}-based
CI test.

\selectlanguage{british}%